\documentclass[a4paper,USenglish,cleveref]{lipics-v2021}
\pdfoutput=1
\hideLIPIcs
\nolinenumbers

\usepackage{amsmath,amssymb,amsfonts}
\usepackage{graphicx}

\usepackage{amsthm}
\usepackage{booktabs}
\usepackage{xspace}
\usepackage{algorithm}
\usepackage[noend]{algpseudocode}
\usepackage[frozencache,cachedir=minted-cache]{minted}
\usepackage{hyperref}
\usepackage{framed}

\usepackage{tikz}
\usetikzlibrary{fit,positioning}

\newcommand{\tool}{\textsc{Bolt}\xspace}
\newcommand{\BoolSynth}{Boolean Set Cover\xspace}
\newcommand{\bigO}{\mathcal{O}}

\renewcommand{\phi}{\varphi}

\newcommand{\vfb}{VFB algorithm\xspace}

\newcommand{\LTLBSSwitch}{\ensuremath{\mathsf{LTL2BS\textup{-}switch}}\xspace}
\newcommand{\dcswitch}{\ensuremath{\mathsf{D\&C\textup{-}switch}}\xspace}
\newcommand{\beamwidth}{\ensuremath{\mathsf{beam\textup{-}width}}\xspace}

\newcommand{\bsSolver}{\textsc{BS-Solver}}
\newcommand{\DandC}{\textsc{DivConq}}

\newcommand{\LTL}{\ensuremath{\textbf{LTL}}\xspace}
\newcommand{\LTLf}{\ensuremath{\textbf{LTL}_f}\xspace}
\newcommand{\AP}{\mathsf{AP}}

\newcommand{\lX}{\mathop{\textbf{X!}}}
\newcommand{\lXweak}{\mathop{\textbf{X}}}
\newcommand{\lG}{\mathop{\textbf{G}}}
\newcommand{\lF}{\mathop{\textbf{F}}}
\newcommand{\lU}{\mathbin{\textbf{U}}}
\newcommand{\lR}{\mathbin{\textbf{R}}}

\newcommand{\set}[1]{\left\{ #1 \right\}}

\newcommand{\suffix}[2]{#1[#2 \dots]}
\newcommand{\size}[1]{|#1|}

\newcommand{\CT}{\mathsf{CT}}

\newcommand{\bs}{\BoolSynth}
\newcommand{\bsSet}{F}
\newcommand{\family}{\mathcal{F}}
\newcommand{\bsFormula}{\theta}
\newcommand{\bsF}{\bsFormula}
\newcommand{\bsMem}{\mathcal{M}}
\newcommand{\bsTop}{\mathcal{T}}
\newcommand{\bsSize}[1]{\mathsf{weight}(#1)}

\usepackage{stmaryrd}
\newcommand{\bsEval}[1]{\llbracket#1\rrbracket}
\newcommand{\bsSat}[1]{\mathsf{sat}( #1 )}
\newcommand{\bsScore}{\mathsf{score}}
\newcommand{\bspq}{\mathsf{pq}}
\newcommand{\bsLim}{b}
\newcommand{\domin}{\preceq}

\newcommand{\op}{\mathsf{op}}

\title{\texorpdfstring{LTL$_f$}{LTLf} Learning Meets Boolean Set Cover}

\author{Gabriel Bathie}{LaBRI, Université de Bordeaux, France \and DIENS, Paris, France}{}{https://orcid.org/0000-0003-2400-4914}{}
\author{Nathanaël Fijalkow}{LaBRI, CNRS, Université de Bordeaux, France}{}{https://orcid.org/0000-0002-6576-4680}{}
\author{Théo Matricon}{LaBRI, Université de Bordeaux, France \and Univ Rennes, CNRS, Inria, IRISA, Rennes, France}{}{https://orcid.org/0000-0002-5043-3221}{}
\author{Baptiste Mouillon}{LaBRI, Université de Bordeaux, France}{}{https://orcid.org/0000-0002-3248-5714}{}
\author{Pierre Vandenhove}{LaBRI, Université de Bordeaux, France \and UMONS -- Université de Mons, Belgium}{}{https://orcid.org/0000-0001-5834-1068}{}

\authorrunning{G. Bathie, N. Fijalkow, T. Matricon, B. Mouillon, P. Vandenhove}
\Copyright{Gabriel Bathie, Nathanaël Fijalkow, Théo Matricon, Baptiste Mouillon, Pierre Vandenhove}

\ccsdesc[500]{Theory of computation~Modal and temporal logics}
\ccsdesc[300]{Computing methodologies~Knowledge representation and reasoning}

\keywords{Linear Temporal Logic, Finite Traces, Specification Mining, \bs}

\newcommand{\artifactLink}{\url{https://doi.org/10.5281/zenodo.18175020}}
\newcommand{\boltLink}{\url{https://github.com/SynthesisLab/Bolt}}
\newcommand{\benchmarkLink}{\url{https://github.com/SynthesisLab/LTLf_Learning_Benchmarks}}
\supplement{
    \tool's code is available at \boltLink.
    The $\LTLf$-learning benchmarks used for the evaluation of our tool are available at \benchmarkLink.
    The artifact for this paper, containing both the above and code to reproduce the experiments, is available at \artifactLink.
}

\funding{This work was partially supported by the SAIF project, funded by the ``France 2030'' government investment plan managed by the French National Research Agency, under the reference ANR-23-PEIA-0006.
Pierre Vandenhove was funded by ANR project G4S (ANR-21-CE48-0010-01).
Experiments presented in this paper were carried out using the Grid'5000 testbed, supported by a scientific interest group hosted by Inria and including CNRS, RENATER and several Universities as well as other organizations (see \url{https://www.grid5000.fr/}).}

\begin{document}

\maketitle

\begin{abstract}
    Learning formulas in Linear Temporal Logic ($\LTLf$) from finite traces is a fundamental research problem which has found applications in artificial intelligence, software engineering, programming languages, formal methods, control of cyber-physical systems, and robotics. We implement a new CPU tool called \tool improving over the state of the art by learning formulas more than $100$x faster over $70$\% of the benchmarks, with smaller or equal formulas in $98$\% of the cases. Our key insight is to leverage a problem called \emph{\bs} as a subroutine to combine existing formulas using Boolean connectives. Thanks to the \bs component, our approach offers a novel trade-off between efficiency and formula size.
\end{abstract}

\newpage
\section{Introduction} \label{sec:intro}
Linear Temporal Logic (\LTL)~\cite{Pnueli:1977} is a prominent logic for specifying temporal properties over infinite traces; in this paper, we consider \LTL on finite traces~\cite{Giacomo.Vardi:2013}, abbreviated $\LTLf$.
The fundamental problem we study is to learn $\LTLf$ formulas from traces: given a set of positive and negative traces, find an $\LTLf$ formula separating positive from negative ones.
$\LTLf$ learning spans different research communities, each contributing applications, approaches, and viewpoints on this problem.
We give below a cursory cross-sectional survey of motivations and applications of $\LTLf$ learning in three communities.

\subparagraph*{Software Engineering, Programming Languages, Formal Methods.}
$\LTLf$~learning is an instantiation of specification mining, which is an active area of research devoted to discovering formal specifications of code.
Already back in the 1970s, Wegbreit~\cite{Wegbreit:1974} and Caplain~\cite{Caplain:1975} propose frameworks to automatically generate properties of code.
At this point, it is important to distinguish between the dynamic and the static setting: in this paper, we are interested in the \emph{dynamic} setting where we observe program executions (also called \emph{traces}) to infer properties of the code.
The term \emph{specification mining} was coined by Ammons, Bod\'ik, and Larus~\cite{Ammons.Bodk.ea:2002} in a seminal paper where finite-state machines capture both temporal and data dependencies.
Zeller~\cite{Zeller:2010} contributed a roadmap to mining specifications in 2010, highlighting its potential for software engineering.
A few years later, Rozier~\cite{Rozier:2016} shaped an entire research program revolving around \emph{\LTL Genesis}, motivating and introducing the \LTL learning problem as we study it here.
Early works focused on mining simple temporal properties~\cite{Engler.Chen.ea:2001}.
In particular, Perracotta~\cite{Yang.Evans.ea:2006} and Javert~\cite{Gabel.Su:2008} focus on patterns of the form~$(a b)^*$ and $(a b^* c)^*$.
The first tool supporting all $\LTLf$ formulas is called Texada, and was created in 2015~\cite{Lemieux.Park.ea:2015,Lemieux.Beschastnikh:2015}.
Following Texada, a line of work focuses on scaling $\LTLf$ learning to industrial sizes, conjuring different approaches. We mention here the state-of-the-art $\LTLf$~learning tools: Scarlet, based on combinatorial search~\cite{Raha.Roy.ea:2022}, and a GPU-accelerated algorithm (later referred to as \emph{\vfb} given its authors Valizadeh, Fijalkow, and Berger)~\cite{Valizadeh.Fijalkow.ea:2024}.
Applications of mining specifications for software include detecting malicious behaviors~\cite{Christodorescu.Jha.ea:2007} or violations~\cite{Li.Zhou:2005} and are already widely adopted in software engineering: for instance, the ARSENAL and ARSENAL2 projects~\cite{Ghosh.Elenius.ea:2016} have been successful in constructing $\LTLf$ formulas inferred from natural language requirements, and the FRET project generates $\LTLf$ from trace descriptions~\cite{Giannakopoulou.Pressburger.ea:2020}.
We refer to the textbook~\cite{Lo.Khoo.ea:2017} and the PhD thesis of Li~\cite{Li:2013} for comprehensive presentation of specification mining for software, and its relationship to data mining.

\subparagraph*{Control of Cyber-Physical Systems and Robotics.}
$\LTLf$ learning is extensively studied in a second area of research with different goals and very different methods: to capture properties of trajectories in models and systems. Given its quantitative nature, Signal Temporal Logic (\textbf{STL}) is preferred over $\LTLf$ since it involves numerical constants and can thus capture real-valued and time-varying behaviors.
In particular, temporal logics allow reasoning about robustness~\cite{Bartocci.Bortolussi.ea:2015} and anomaly detection~\cite{Jones.Kong.ea:2014,Kong.Jones.ea:2017}.
There is a vast body of work on learning temporal logics in control and robotics, which can be divided into two: techniques aiming at fitting parameters of a fixed assumed \textbf{STL} formula~\cite{Asarin.Donze.ea:2012,Yang.Hoxha.ea:2012,Jin.Donze.ea:2015}, and approaches searching for both formulas and parameters~\cite{Jin.Donze.ea:2015,Bombara.Vasile.ea:2016}.
This led to many case studies: automobile transmission controller and engine airpath control~\cite{Yang.Hoxha.ea:2012},
assisted ventilation in intensive care patients~\cite{Bufo.Bartocci.ea:2014},
dynamics of a biological circadian oscillator and discriminating different types of cardiac malfunction from electro-cardiogram data~\cite{Bartocci.Bortolussi.ea:2014},
anomaly detection in a maritime environment~\cite{Bombara.Vasile.ea:2016},
demonstrations in robotics~\cite{Chou.Ozay.ea:2020},
detection of attention loss in pilots~\cite{Lyu.Li.ea:2024},
and analysis of computer games~\cite{Guti_rrez_S_nchez_2025}.

\subparagraph*{Artificial Intelligence.}
The field of AI has contributed to $\LTLf$ learning a number of applications, but also a wealth of techniques.
The overarching philosophy of $\LTLf$ learning in AI is that $\LTLf$ forms a natural and explainable formalism~\cite{Camacho.McIlraith:2019} for specifying objectives~\cite{Icarte.Klassen.ea:2018} in various machine learning contexts. For instance, in the context of reinforcement learning, $\LTLf$ formulas have been used to either guide~\cite{Camacho.Chen.ea:2017,Brafman.Giacomo.ea:2018} or constrain policies~\cite{Icarte.Klassen.ea:2018*1,Camacho.Icarte.ea:2019}. Another recent application aims to speed up synthesis by separating data and control~\cite{Murphy.Holzer.ea:2024}.

\subparagraph*{SOTA on $\LTLf$ Learning.}
Since learning temporal logics is computationally hard\footnote{It is NP-complete, even for restricted fragments of $\LTLf$~\cite{Fijalkow.Lagarde:2021,Mascle.Fijalkow.ea:2023}.}, different approaches have been explored.
The state-of-the-art tool is Scarlet, which is based on specifically tailored fragments of $\LTLf$~\cite{Raha.Roy.ea:2022,Raha.Roy.ea:2024}.
A GPU-accelerated algorithm was published in~2024~\cite{Valizadeh.Fijalkow.ea:2024}, yielding an improved state of the art in a different category.
Scarlet cannot reliably learn formulas of size greater than~$15$, which is less than ideal: we aim to change this.

\subparagraph*{\bs.}
Our key insight in this work, and what we believe is the key to scaling $\LTLf$ learning to industrial sizes, is to treat Boolean operators differently from temporal operators. To address the former, we leverage a novel solution to a problem we call \emph{\BoolSynth}~\cite{Raha.Roy.ea:2022,Roy2025}.
When enumerating $\LTLf$ formulas, we quickly hit a wall because of the exponential number of formulas. Let us say that at this point, we have enumerated the formulas $\phi_1,\dots,\phi_k$ but have not found a solution. Instead of giving up, \bs looks for solutions in the form of Boolean combinations of the formulas $\phi_1,\dots,\phi_k$, allowing us to generate much larger and more expressive formulas. See Figure~\ref{fig:BS}.

\begin{figure*}[t]
	\centering
	\includegraphics[width=0.999\textwidth]{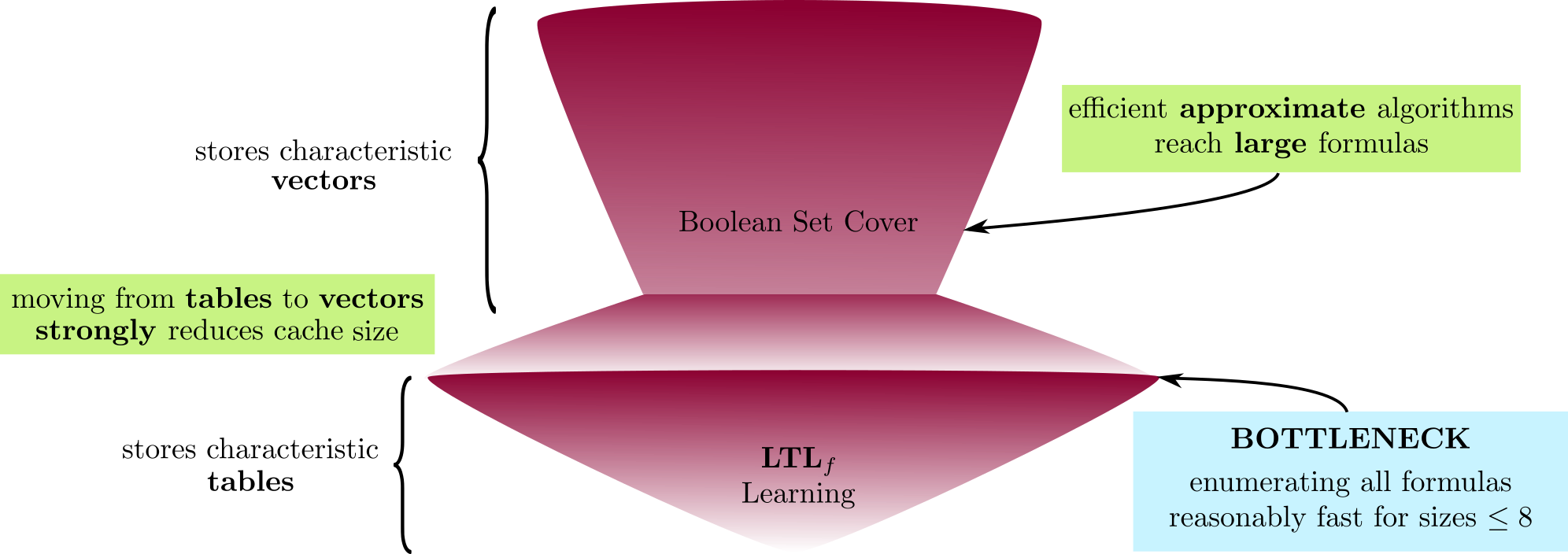}
	\caption{Overview of $\LTLf$ Learning with \bs}
	\label{fig:BS}
\end{figure*}

\bs is a fundamental problem which has been studied under different names. In the $80$s, when key concepts from learning theory emerged, it was known as \emph{Boolean Concept learning}~\cite{Valiant_1984,Angluin_1988} (see also the textbook~\cite{anthony1997computational}).
In logic, it has been extensively studied as \emph{extending partially defined Boolean functions}; see, e.g.,~\cite{Crama_1988}.
The name \emph{\BoolSynth} was introduced in~\cite{Raha.Roy.ea:2022} in the context of $\LTLf$ learning, highlighting that it extends the classical \emph{Set Cover} problem. For consistency in this line of work, we adhere to this terminology.

The \textbf{key finding} of this work is that the \BoolSynth problem can be solved \textit{in an approximate way} very efficiently, which can then be leveraged to learn $\LTLf$ formulas of size way beyond the reach of existing tools. An overview of our general approach is depicted in Figure~\ref{fig:BS}.

\subparagraph*{Our Contributions:}
\begin{itemize}
    \item We propose a framework for combining $\LTLf$ learning with \bs.
    \item We implement our algorithms in a new tool called $\tool$.
    \item We consolidate a benchmark suite for $\LTLf$ learning with over $15,\!000$ tasks.
	\item We show through experiments that $\tool$ significantly improves over the state of the art, both in terms of wall-clock time and size of learned formulas.
\end{itemize}

We believe \BoolSynth is a fundamental problem of independent interest, which has a lot of potential applications beyond $\LTLf$ learning. We leave as future work to explore them; we discuss perspectives in Section~\ref{sec:perspectives}.

\subparagraph*{Code Availability.}
An artifact with all the code, dataset, and additional scripts to reproduce our experiments is available at \artifactLink.
\tool's code is available at \boltLink.
The benchmarks used for the evaluation of our tool are available at \benchmarkLink.

\subparagraph*{Conference Version.}
This article extends a conference version~\cite{conferenceVersion} in the proceedings of TACAS~2026 with the complete proofs and additional examples and details.

\section{\texorpdfstring{$\LTLf$}{LTLf} Learning} \label{sec:ltl_learning}
\subsection{Problem Definition}

One of the reasons for the success of \emph{Linear Temporal Logic} ($\LTLf$) as a logical formalism for temporal reasoning on traces is that its semantics can be conveyed using a single picture:

\begin{center}
\includegraphics[width=0.65\textwidth]{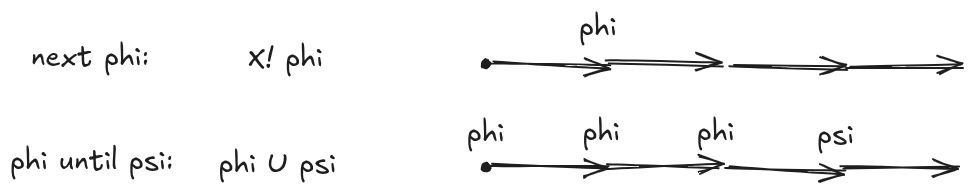}
\end{center}

The syntax of $\LTLf$ is particularly simple (no variables and no quantifiers), it uses only the classical \textit{Boolean operators} and two \textit{temporal operators}: $\lX$, called \textit{Strong Next}, and $\lU$, called \textit{Until}.\footnote{The temporal operators $\lF$ and $\lG$ will be derived from~$\lU$.}
Informally, $\lX \phi$ holds if $\phi$ holds starting from the next position,
and $\phi \lU \psi$ holds if there exists a later position such that $\psi$ holds, and $\phi$ holds in the meantime.

Formally, let us fix a finite set of \emph{atomic propositions} $\AP$.
A (finite) trace is a (non-empty) sequence where each position holds a subset of atomic propositions; for instance, with~$\AP = \{p, q\}$:
\begin{center}
\includegraphics[width=0.6\textwidth]{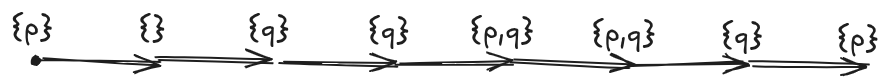}
\end{center}

We index traces from position $1$ (not $0$) and the letter at position $i$ in the trace $w$ is written $w(i)$,
so $w = w(1) \dots w(\ell)$ where $\ell$ is the length of $w$, written $|w| = \ell$.
We write $\suffix{w}{k} = w(k) \dots w(\ell)$.
We only consider non-empty traces.

The syntax of \emph{Linear Temporal Logic} ($\LTLf$) includes atomic propositions $c \in \AP$ as well as $\top$ and $\bot$, the Boolean operators $\lnot$, $\wedge$, and $\vee$, and the \emph{temporal operators} $\lX$ and~$\lU$.
The semantics of $\LTLf$ are defined inductively, through the notation $w \models \phi$ where $w$ is a non-empty trace and $\phi$ is an $\LTLf$ formula.
The definition is given below for the atomic propositions and temporal operators $\lX$ and $\lU$, with Boolean operators interpreted as usual.
For $w$ of length~$\ell$:
\begin{itemize}
	\item $w \models \top$ and $w \not\models \bot$.
    \item $w \models c$, with $c\in\AP$, if $c\in w(1)$.
    \item $w \models \lX \phi$ if $\ell > 1$ and $\suffix{w}{2} \models \phi$ (\emph{strong next}).
    \item $w \models \phi \lU \psi$ if there is $i \in [1,\ell]$ such that for $j \in [1,i-1]$ we have $\suffix{w}{j} \models \phi$,
    and $\suffix{w}{i} \models \psi$ (\emph{until}).
\end{itemize}
We say that \emph{$w$ satisfies~$\phi$} if $w \models \phi$.
We define the \emph{size} of an $\LTLf$ formula $\phi$ as the number of nodes in its syntactic tree.

As customary, we include some additional temporal operators:
\begin{itemize}
    \item $\lXweak \phi$ is ``weak next~$\phi$'', it is a shortcut for $\lnot \lX \lnot \phi$;
    \item $\lF \phi$ is ``eventually~$\phi$'', it is a shortcut for $\top \lU \phi$;
    \item $\lG \phi$ is ``globally~$\phi$'', it is a shortcut for $\lnot \lF \lnot \phi$;
    \item $\phi \lR \psi$ is ``release~$\phi$'', it is a shortcut for $\lnot \psi \lU \lnot \phi$.
\end{itemize}

The $\LTLf$ learning problem is defined as follows:
\begin{framed}
\centering
\begin{tabular}{p{1.9cm}p{10cm}}
\normalsize\textbf{INPUT}:& \normalsize two disjoint finite sets of traces $P$ and $N$ \\
\normalsize\textbf{OUTPUT}:& \normalsize an $\LTLf$ formula $\phi$ of minimal size such that all of $P$ satisfies $\phi$ and none of $N$ satisfies $\phi$
\end{tabular}
\end{framed}

Note that the minimality of a formula depends on the set of operators considered (both Boolean and temporal).
The set of operators we consider is standard for $\LTLf$.

\subsection{The VFB Algorithm} \label{secsub:ltl_implem}
Our first baseline for $\LTLf$ learning is an optimized enumerative algorithm from Valizadeh, Fijalkow, and Berger~\cite{Valizadeh.Fijalkow.ea:2024}; we call it the \emph{\vfb}. Simply put:
\begin{center}
\vfb = combinatorial search + fast evaluation + observational equivalence.
\end{center}

\subparagraph*{Combinatorial Search.}
The core algorithm is a bottom-up enumeration of $\LTLf$ formulas by size.
First, the set of formulas of size 1 contains all atomic propositions.
At step~$k$, we generate all formulas of size~$k + 1$.
In order to do this, for each operator, we take arguments such that the sum of sizes of arguments is equal to $k$; this way, the generated formula will have size $k+1$ by combining it with the selected operator.
A~property of this algorithm is that it is agnostic to the operators: it can be adapted to any fragment or be augmented with other operators, as long as we can easily compute inductively whether a trace satisfies a formula.

\subparagraph*{Fast Evaluation.}
The key idea is to represent the semantics of $\LTLf$ formulas on input traces using so-called ``characteristic tables'':
\begin{itemize}
    \item The \emph{characteristic sequence (CS)} of a formula $\phi$ on a trace $w$ is the bit vector $v\in\set{0, 1}^{\size{w}}$ such that $w[i\ldots] \models \phi$ if and only if $v(i)=1$;
    \item The \emph{characteristic table (CT)} of a formula $\phi$ over the positive traces $P$ and the negative traces~$N$ is a sequence~$t$ such that $t(i)$ is the CS of $\phi$ on the $i$\textsuperscript{th} trace of $P \cup N$.
\end{itemize}
We give an example of characteristic sequences and tables in Figure~\ref{fig:charseq}.

\definecolor{green}{rgb}{0,0.6,0}
\definecolor{red}{rgb}{.8,0,0}
\definecolor{blue}{rgb}{0,0,0.8}
\begin{figure}[th]
    \centering
    \begin{tikzpicture}[scale=1.1]
        \node (p1) at (0, 0) {$aabaa$};
        \node (p2) at (0, -.6) {$baaa$};
        \node (n3) at (0, -1.2) {$abab$};
        \node (n4) at (0, -1.8) {$aab$};

        \node[right of=p1,node distance=1.2cm] (11) {$1$};
        \node[right of=p2,node distance=1.2cm] (12) {$1$};
        \node[right of=n3,node distance=1.2cm] (13) {$0$};
        \node[right of=n4,node distance=1.2cm] (14) {$1$};

        \node[right of=11,node distance=.7cm] (21) {$0$};
        \node[right of=12,node distance=.7cm] (22) {$1$};
        \node[right of=13,node distance=.7cm] (23) {$1$};
        \node[right of=14,node distance=.7cm] (24) {$0$};

        \node[right of=21,node distance=.7cm] (31) {$1$};
        \node[right of=22,node distance=.7cm] (32) {$1$};
        \node[right of=23,node distance=.7cm] (33) {$0$};
        \node[right of=24,node distance=.7cm] (34) {$0$};

        \node[right of=31,node distance=.7cm] (41) {$1$};
        \node[right of=32,node distance=.7cm] (42) {$0$};
        \node[right of=33,node distance=.7cm] (43) {$0$};

        \node[right of=41,node distance=.7cm] (51) {$0$};

		\node[draw, green, thick, fill=green!30, fill opacity=.4, rounded corners=8pt, inner sep=3pt, fit=(11)(14)] {};
        \node[below of=14,node distance=.7cm] {\color{green}Characteristic vector};

        \node[draw, blue, thick, fill=blue!30, fill opacity=.4, rounded corners=8pt, inner sep=4.5pt, fit=(11)(51)] {};
        \node[above of=31,node distance=.7cm] {\color{blue}Characteristic sequence};

        \node[draw, red, thick, fill=red!30, fill opacity=.25, rounded corners=12pt, inner sep=7pt, fit=(11)(51)(14)(34), label=right:{\color{red}Characteristic table}] {};
    \end{tikzpicture}
    \caption{Characteristic \emph{sequences} and \emph{table} of $\phi = \lX a$ with $P = \{aabaa, baaa\}$ and $N = \{abab, aab\}$.
    We also use \emph{characteristic vector} for the vector of all first bits of the characteristic sequences.
    Looking at the characteristic vector suffices to check whether $\phi$ is a solution: it must have $1$'s at all positions of $P$ and $0$'s at all positions of $N$.
    However, we need to keep in memory the whole characteristic table to check for observational equivalence.}
    \label{fig:charseq}
\end{figure}
The characteristic table is not a rectangular matrix when traces in $P \cup N$ do not have the same length.

With these objects, one can compute inductively the characteristic sequences and tables of formulas with only a few bitwise operations.
For instance, the semantic of the strong next operator $\lX$ is simply a ``left shift'' on characteristic sequences.
This idea originates from~\cite{Valizadeh.Fijalkow.ea:2024}.
The brief Python code to compute efficiently the characteristic sequences of $\LTLf$ formulas is recalled in Appendix~\ref{app:pseudocodesLTL}.

\subparagraph*{Observational Equivalence.}
A second idea to tackle the combinatorial explosion of the number of formulas is to consider only ``semantically unique'' formulas.
\emph{Observational equivalence} is an idea from program synthesis~\cite{Gabbrielli.Levi.ea:1992}; in our context, two formulas are \emph{observationally equivalent} on $(P, N)$ if they generate the same characteristic table.
When enumerating formulas, we may keep a single representative for each equivalence class without losing completeness.

\subparagraph*{Scaling Issues.}
Despite these improvements to a simple exhaustive search, the scaling of the problem is enormous.
We identified two problems limiting the depth of algorithms.

The first issue, of course, is the exponential blow-up due to the number of increasingly large formulas.
The second issue is the memory used to store the evaluation of \LTLf formulas in the form of characteristic tables, which is necessary to check for observational equivalence.
The more input traces there are, the larger the characteristic tables are; the longer the traces are, the wider the tables are.
When tables have more elements, observational equivalence is also less likely to occur.

\section{\bs} \label{sec:bs}
Let us start with defining \bs formally. We consider a finite set $X$ partitioned into positive and negative elements: $X = P \cup N$. We are given a (potentially large) set of atomic formulas $\phi_1,\dots,\phi_k$ over $X$. The goal is to construct a formula separating $P$ from $N$ using conjunctions and disjunctions of the atomic formulas. 

As an example, consider the following 
instance of \BoolSynth (adapted from~\cite[Fig.~1]{Raha.Roy.ea:2022}):
\begin{center}
\includegraphics[width=0.35\columnwidth]{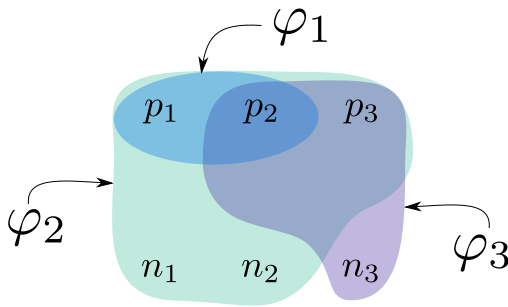}
\end{center}
Here, $P = \set{p_1,p_2,p_3}$ and $N = \set{n_1,n_2,n_3}$. The formulas $\varphi_1, \varphi_2$, and $\varphi_3$ satisfy the points encircled in the corresponding area. In this instance, $\varphi_1 \vee (\varphi_2 \wedge \varphi_3)$ is a (minimal) solution.

Our idea to tackle these scaling issues is to use the \vfb on the full $\LTLf$ only up to a fixed formula size, chosen based on a memory or time threshold.
Since the exponential blow-up does not allow to go much further beyond size~$8$, we then restrict our focus to the fragment of $\LTLf$ with only $\land$ and $\lor$ to move forward.
In other words, we keep all formulas we have generated and then only combine them with~$\land$ or~$\lor$.
This restricted problem is an instance of \emph{\bs}, that is solved as a subroutine in our $\LTLf$ learning algorithms.

A high-level view of the procedure is in Algorithm~\ref{algo:combinatorial_search_with_boolean_set_cover}.
In this algorithm, \textsc{VFB-Bounded} is a variant of \textsc{VFB} which only computes $\LTLf$ formulas up to some fixed size using operators from sets of unary and binary operators $O_1$ and~$O_2$.
This fixed size is a hyperparameter called~$\LTLBSSwitch$.
Algorithm \textsc{VFB-Bounded} returns two objects: a solution to the~$\LTLf$ learning instance (or $\bot$ if it has not found a solution), and a set $\mathcal{F}$ containing all generated formulas up to size $\LTLBSSwitch$.
If no solution was found, the generated formulas are fed to a \bs solver called \bsSolver.
From the point of view of the \bs solver, the $\LTLf$ formulas are letters of a new alphabet, each being associated with a positive integer called \emph{weight} (which is in practice the size of the $\LTLf$ formula).

\begin{algorithm}
	\small
	\caption{\vfb with \bs solver}\label{algo:combinatorial_search_with_boolean_set_cover}
	\textbf{Hyperparameter}: $\LTLBSSwitch$ (below, $k$), the maximal size of \LTLf formulas to enumerate before switching to \bs.
	\begin{algorithmic}[1]
		\Procedure{Search}{$\AP, O_1, O_2, P, N, k$}
		\State $S,\family \gets \textsc{VFB-Bounded}(\AP, O_1, O_2, k)$
		\If{$S \neq \bot$} \Comment{If the \vfb found a solution}
		\State \textbf{return} $S$
		\Else
		\State $\family' \gets$ \textsc{Collapse}($\family$)
		\State \textbf{return} $\bsSolver(\family', P, N)$
		\EndIf
		\EndProcedure
	\end{algorithmic}
\end{algorithm}

When we send the generated formulas to a \bs solver, two kinds of collapses occur (represented in Algorithm~\ref{algo:combinatorial_search_with_boolean_set_cover} with function \textsc{Collapse}).

First, once we restrict our focus to operators $\land$ and $\lor$, two formulas can be assumed to be observationally equivalent if and only if they have the same \emph{first bit} for each row of their characteristic table.
Indeed, as there is no more ``temporal phenomenon'', this is a sufficient information to determine whether a given trace satisfies a Boolean combination of (already evaluated) \LTLf formulas.
In other words, we move from the computation of \emph{tables} to \emph{vectors} (the vectors of first bits), which produces many collisions and reduces the number of formulas to consider.
We quantified this collapse in practice over a small set of tasks chosen randomly.
Let $\family$ denote the set of observationally non-equivalent $\LTLf$ formulas of size at most $8$, and $\family' = $ \textsc{Collapse}($\family$) denote the set of formulas with a distinct first column in the characteristic table.
On average over $780$ tasks, the ratio $\size{\family} / \size{\family'}$ is $3.54$, with a minimum of $1.09$, a maximum of $48.22$, and a standard deviation of $5.46$.

Second, storing vectors instead of tables is a massive gain: space-wise, as storing the evaluation of a formula is a bit vector of size $|P| + |N|$ instead of~$\ell \cdot (|P| + |N|)$, where $\ell$ is the total number of bits in an input trace, and time-wise, as new formulas can be evaluated faster.
This enables exploring much larger formulas (though in an incomplete way, as some operators are missing) at a reduced space and time cost.

\subsection{Definition and Properties of \bs} \label{secsub:bsIntro}
Let $P$ and $N$ be two disjoint finite sets.
Let $\family$ be a family of sets with $\bsSet \subseteq P \cup N$ for all~$\bsSet\in\family$.
We assume that a \emph{weight} $\bsSize{\bsSet}\ge 1$ is associated with each set $\bsSet\in\family$ (which is unrelated to the cardinality of $\bsSet$).

We define \emph{positive Boolean formulas} and their weights inductively:
\begin{itemize}
	\item $\emptyset$ is a formula of weight $0$;
	\item $\bsSet\in\family$ is a formula of weight $\bsSize{\bsSet}$;
	\item if $\bsFormula_1$ and $\bsFormula_2$ are positive Boolean formulas, then so are $\bsFormula_1 \cup \bsFormula_2$ and $\bsFormula_1 \cap \bsFormula_2$, both of weight $1 + \bsSize{\bsFormula_1} + \bsSize{\bsFormula_2}$.
\end{itemize}
The interpretation~$\bsEval{\bsFormula}$ of a positive Boolean formula $\bsFormula$ is simply the set it describes.
We adjust observational equivalence: two formulas $\bsFormula$ and $\bsFormula'$ are observationally equivalent if $\bsEval{\bsFormula} = \bsEval{\bsFormula'}$.

The \bs problem is as follows:

\begin{framed}
\centering
\begin{tabular}{p{1.9cm}p{10cm}}
	\normalsize\textbf{INPUT}:& \normalsize two disjoint finite sets $P$, $N$ and a family $\family$ of weighted subsets of~$P \cup N$\\
	\normalsize\textbf{OUTPUT}:& \normalsize a positive Boolean formula $\bsFormula$ of minimal weight such that $\bsEval{\bsFormula} = P$ (i.e., it contains all elements of~$P$ but none of~$N$)
\end{tabular}
\end{framed}

The \bs problem can be seen as a restriction of $\LTLf$ learning to a subset of the operators, namely $\land$ and $\lor$.
The \bs problem was briefly considered in~\cite{Raha.Roy.ea:2022} (under the name \emph{Boolean Subset Cover}).
We will compare our algorithm to theirs in Section~\ref{sec:experiments}.

\subparagraph*{Polynomial Criterion for a Solution of Arbitrary Weight.}
We show that we can check whether there exists a formula $\bsFormula$ (\emph{of arbitrary weight}) such that $\bsEval{\bsFormula} = P$ in polynomial time, just by evaluating a specific formula of length $\bigO(|P|\cdot |\family|)$.
This allows to quickly decide the nonexistence of a solution, and when a solution exists, it gives a polynomial upper bound on the weight of formulas to consider.

\begin{lemma}[Existence of a solution] \label{lem:solution_exists}
	Let $P$, $N$, $\family$ be an instance of \bs.
	There exists a positive Boolean formula $\bsFormula$ such that $\bsEval{\bsFormula} = P$ if and only if for all $p\in P$, $n\in N$, there is $\bsSet\in\family$ such that $p\in \bsSet$ and $n\notin \bsSet$.
	Moreover, when it exists, there is such a formula of size $\bigO(|\family|\cdot |P|)$.
\end{lemma}

\begin{proof}
	For the left-to-right implication, we show the contrapositive.
	Assume there is $p\in P$ and $n\in N$ such that, for all $\bsSet\in\family$, $p\in\bsSet$ implies $n\in\bsSet$.
	One can then show by induction that any formula~$\bsFormula'$ with $p\in\bsEval{\bsFormula'}$ also has $n\in\bsEval{\bsFormula'}$.

	For the right-to-left implication, we build a positive Boolean formula $\bsFormula$ such that $\bsEval{\bsFormula} = P$.
	This formula happens to be of size $\bigO(|\family|\cdot |P|)$, thereby also settling the claim about the size.
	Assume that for all $p\in P$, $n\in N$, there is $\bsSet\in\family$ such that $p\in \bsSet$ and $n\notin \bsSet$.
	For~$p\in P$, let $\bsFormula_p = \bigcap_{\bsSet\in\family, p \in \bsSet} \bsSet$ and $\bsFormula = \bigcup_{p\in P} \bsFormula_p$.
	Due to the hypothesis, for all $p\in P$, we have $p\in \bsEval{\bsFormula_p}$ and $N \cap \bsEval{\bsFormula_p} = \emptyset$.
	Hence, $\bsEval{\bsFormula} = P$.
\end{proof}

\subparagraph*{Complexity.}
\bs is a generalization of the NP-complete \emph{set cover} problem~\cite{Karp:1972}: an instance of the set cover problem corresponds to an instance of \bs with~$N = \emptyset$ (in which case using $\cap$ is futile, and minimal formulas can be obtained only using $\cup$).
This proves NP-hardness of the decision problem for \bs.
Since there are short solutions by Lemma~\ref{lem:solution_exists}, and since computing the set described by a Boolean formula is polynomial in its length, \bs is NP-complete.

\subsection{Domination for positive Boolean Formulas} \label{secsub:domination}
In this section, we describe \emph{domination}, a generalization of observational equivalence that identifies even more redundant objects than observational equivalence.

Intuitively, we say that a formula $\bsF_1$ \emph{dominates} a formula~$\bsF_2$ if $\bsF_2$ can be replaced with $\bsF_1$ in any formula without increasing its weight or decreasing its ``quality''.
More formally, let $\bsSat{\bsF}$ be the subset of elements of $P \cup N$ that $\bsF$ classifies correctly, i.e.,
$\bsSat{\bsF} = (\bsEval{\bsF} \cap P) \cup (N \setminus \bsEval{\bsF})$.
The goal of \bs can then be rephrased as finding a formula $\bsF$ such that $\bsSat{\bsF} = P \cup N$.
\begin{definition}
	\label{def:domination}
	A formula $\bsF_2$ is \emph{dominated} by a formula $\bsF_1$, denoted $\bsF_2 \domin \bsF_1$, if $\bsSize{\bsF_1} \le \bsSize{\bsF_2}$ and $\bsSat{\bsF_2} \subseteq \bsSat{\bsF_1}$.
\end{definition}

We claim that this notion corresponds to the above intuition.
Formally, for a positive Boolean formula $\bsF$, let $\bsF[\bsF_2 \leftarrow \bsF_1]$ denote the formula obtained by replacing any occurrence of $\bsF_2$ in $\bsF$ by $\bsF_1$; this operation is defined inductively as follows:
\begin{align*}
	F_i[\bsF_2 \leftarrow \bsF_1] &= \begin{cases}
		\bsF_1& \text{if } F_i = \bsF_2\\
		F_i& \text{otherwise,}
	\end{cases}\\
	\bsF_2[\bsF_2 \leftarrow \bsF_1] &= \bsF_1,\\
	(\bsF \cap \bsF')[\bsF_2 \leftarrow \bsF_1] &= \bsF[\bsF_2 \leftarrow \bsF_1] \cap \bsF'[\bsF_2 \leftarrow \bsF_1],\\
	\text{and }(\bsF \cup \bsF')[\bsF_2 \leftarrow \bsF_1] &= \bsF[\bsF_2 \leftarrow \bsF_1] \cup \bsF'[\bsF_2 \leftarrow \bsF_1].
\end{align*}
Any occurrence of a dominated formula can be replaced with its dominating formula:
\begin{lemma} \label{lemma:domin-works}
	Let $\bsF$, $\bsF_1$, and $\bsF_2$ be positive Boolean formulas.
	If $\bsF_2 \domin \bsF_1$, then $\bsF \domin \bsF[\bsF_2 \leftarrow \bsF_1]$.
\end{lemma}
\begin{proof}
	First, note that, as $P$ and $N$ are disjoint, then if $\bsSat{\bsF_2} \subseteq \bsSat{\bsF_1}$ we have:
	\begin{equation}
		\bsEval{\bsF_2} \cap P \subseteq \bsEval{\bsF_1} \cap P \text{ and } N\setminus \bsEval{\bsF_2} \subseteq N\setminus\bsEval{\bsF_1}.\label{eq:inclusion}
	\end{equation}

	We now show that the $\bsSat{\cdot}$ function is \emph{monotone} with respect to the $\cup$ and $\cap$ operators:
	if $\bsSat{\bsF_2} \subseteq \bsSat{\bsF_1}$, then for any formula $\bsF$, we have $\bsSat{\bsF_2 \cap \bsF} \subseteq \bsSat{\bsF_1 \cap \bsF}$ and $\bsSat{\bsF_2 \cup \bsF} \subseteq \bsSat{\bsF_1 \cup \bsF}$.
	This can be shown through direct calculation using \cref{eq:inclusion}:
	\begin{align}
		\bsSat{\bsF_2 \cap \bsF}
		&= (\bsEval{\bsF_2 \cap \bsF} \cap P) \cup (N \setminus \bsEval{\bsF_2 \cap \bsF})\notag\\
		&= (\bsEval{\bsF_2} \cap \bsEval{\bsF} \cap P) \cup (N \setminus (\bsEval{\bsF_2} \cap \bsEval{\bsF}))\notag\\
		&= ((\bsEval{\bsF_2} \cap P) \cap \bsEval{\bsF}) \cup ((N \setminus \bsEval{\bsF_2}) \cup (N\setminus \bsEval{\bsF}))\notag\\
		&\subseteq ((\bsEval{\bsF_1} \cap P) \cap \bsEval{\bsF}) \cup ((N \setminus \bsEval{\bsF_1}) \cup (N\setminus \bsEval{\bsF}))\label{line:incl}\\
		&= (\bsEval{\bsF_1} \cap \bsEval{\bsF} \cap P) \cup (N \setminus (\bsEval{\bsF_1} \cap \bsEval{\bsF}))\notag\\
		&= \bsSat{\bsF_1 \cap \bsF}.\notag
	\end{align}
	The inclusion in eq.~\eqref{line:incl} follows from eq.~\eqref{eq:inclusion}.
	A similar calculation shows the inclusion for the $\cup$ operator.

	Similarly, the $\bsSize{\cdot}$ function is also monotone: if $\bsSize{\bsF_1} \le \bsSize{\bsF_2}$, then
	$\bsSize{\bsF_1 \cap \bsF} \le \bsSize{\bsF_2 \cap \bsF}$ (and similarly for
	$\cup$) for any formula $\bsF$.

	We can now prove that if $\bsF_2 \domin \bsF_1$, then $\bsF \domin \bsF[\bsF_2 \leftarrow \bsF_1]$ by structural induction over~$\bsF$.
	The base cases follow from the definition of domination, and the cases $\bsF = \bsF' \cap \bsF''$ or $\bsF = \bsF' \cup \bsF''$ are handled using the monotonicity of the $\bsSat{\cdot}$ and $\bsSize{\cdot}$ functions.
\end{proof}

Domination can be thought of as a variation on \emph{Property dependence} (see Theorem 6 in~\cite{Dureja.Rozier:2018}), used to prune $\LTLf$ formulas for efficient model-checking.

\subparagraph*{Further Reducing the Input Set.}
\Cref{lemma:domin-works} implies that we can preprocess $\family$ to only keep maximal elements for~$\domin$ (which means that we transform the input into an antichain for~$\domin$).
By encoding the characteristic vectors of formulas in bits of integers, testing whether $\bsF_1$ dominates $\bsF_2$ can be implemented efficiently using bitwise operations.
All our algorithms for \bs start by reducing $\family$ in this way.

\subparagraph*{The Need for a Heuristic.}
We argue that the existence of a subquadratic algorithm to completely reduce a set $\family$ with domination is unlikely.
The obvious algorithm to find whether a given formula $\bsF\in\family$ is dominated by a formula $\bsF'\in\family$ is to go through every formula in $\family$ and check whether it dominates~$\bsF$.
This takes linear time,
and when there are $O(n)$ formulas~$\bsF$ and~$\bsF'$, this approach takes quadratic time.
We do not expect a faster algorithm:
the \emph{Orthogonal Vectors problem}~\cite{Williams:2005} can be reduced to whether there exists a formula~$\bsFormula_1\in A$ that dominates a formula $\bsFormula_2 \in B$ where $A$, $B$ are sets of $n$ formulas,
and a strongly subquadratic algorithm for Orthogonal Vectors would imply that the Strong Exponential Time Hypothesis is false.
The reduction between our problem and the Orthogonal Vectors problem is presented in~\cite{Borassi.Crescenzi.ea:2015}.

\subparagraph*{Fast Approximate Domination.}
Yet, a quadratic running time is prohibitive in our application where we need to test domination for millions of Boolean formulas:
we therefore propose a heuristic that drastically speeds up the search of a dominating formula while allowing for some false negatives.

Instead of searching through all of $\family$ for dominating formulas, we restrict the search to a small subset of candidates
with a high likelihood of dominating other formulas.
For a given weight value $w$ and an integer $k$, let $\bsTop(w, k)$ denote the $k$ formulas $\bsF$ that maximize $\size{\bsSat{\bsF}}$ among formulas of weight $w$ in~$\family$.
Our heuristic searches for a formula $\bsF$ that dominates $\bsF'$ in the sets $\bsTop(w, k)$ for $w \le \bsSize{\bsF'}$.
The value chosen for $k$ controls the trade-off between speed and accuracy of the technique.
We illustrate the gain in practice over a small set of tasks chosen randomly, which is substantial even for small values of $k$.
Let $\family'$ be as in Algorithm~\ref{algo:combinatorial_search_with_boolean_set_cover}, and $\family'' = \textsc{FastNonDominated}(\family', k)$ be the result of the shrinkage of $\family'$ by the above approximate algorithm.
On average, over $780$ tasks, the ratio $\size{\family'}/\size{\family''}$ is $3.45$ for~$k = 3$, $4.07$ for $k = 5$, $4.97$ for $k = 10$, $5.94$ for~$k = 25$, and $6.60$ for $k = 50$.

\subsection{Divide and Conquer} \label{secsub:divideConquer}
Before describing our algorithms, we show a general procedure to extend any solver for \bs to mitigate its inability to find a valid solution (due to time, space, or incompleteness of the algorithm).
The general idea is part of the $\LTLf$ learning algorithm from~\cite{Valizadeh.Fijalkow.ea:2024}.
It is an application of \emph{divide and conquer}: when a \bs solver fails to find a solution, divide $P$ or $N$ in two smaller subsets, solve the two subproblems, and combine them into a solution.
We show the complete pseudocode for the divide-and-conquer algorithm in Algorithm~\ref{algo:divideAndConquer}.

\begin{algorithm*}[tbph]
	\caption{Divide and conquer (meta-algorithm for \bs)}\label{algo:divideAndConquer}
	\begin{algorithmic}[1]
		\Procedure{\DandC}{$\family, P, N$}
		\While{$|P| > 1$ or $|N| > 1$}
		\State{$\bsFormula \gets$ \bsSolver($\family, P, N$)}
		\If{$\bsEval{\bsFormula} = P$} \Comment{Valid solution}
		\State{\Return{$\bsFormula$}}
		\EndIf
		\If{$|P| \ge |N|$}
		\State{$P_1, P_2 \gets$ split $P$ into two subsets of roughly equal cardinality}
		\State{$\bsFormula_1 \gets$ \DandC($\family_{\restriction P_1\cup N}, P_1, N$)}
		\State{$\bsFormula_2 \gets$ \DandC($\family_{\restriction P_2\cup N}, P_2, N$)}
		\State{\Return{$\bsFormula_1 \cup \bsFormula_2$}}
		\Else
		\State{$N_1, N_2 \gets$ split $N$ into two subsets of roughly equal cardinality}
		\State{$\bsFormula_1 \gets$ \DandC($\family_{\restriction P\cup N_1}, P, N_1$)}
		\State{$\bsFormula_2 \gets$ \DandC($\family_{\restriction P\cup N_2}, P, N_2$)}
		\State{\Return{$\bsFormula_1 \cap \bsFormula_2$}}
		\EndIf
		\EndWhile
		\State{take $p\in P$, $n\in N$} \Comment{Here, $P = \{p\}$ and $N = \{n\}$}
		\State{\Return{$F\in\family$ of minimal weight s.t.\ $p\in F$, $n\notin F$}}\label{line:baseCaseDC}
		\EndProcedure
	\end{algorithmic}
\end{algorithm*}

The meta-algorithm first calls a \bs solver; we assume that this algorithm always returns a formula, but which may not be a valid solution.
In case no solution is found, the problem is broken into two subproblems.
If $|P| \ge |N|$, then $P$ is randomly split into two subsets $P_1$, $P_2$ of roughly equal cardinality.
From this, we consider two smaller instances of \bs: $(P_1, N, \family_{\restriction P_1\cup N})$ and $(P_2, N, \family_{\restriction P_2\cup N})$, where $\family_{\restriction U} = \{ F\cap U\mid F\in\family\}$.
The two subproblems are solved recursively, yielding respectively solutions $\bsFormula_1$ and~$\bsFormula_2$; observe that~$\bsFormula = \bsFormula_1 \cup \bsFormula_2$ is a solution to the original problem.
If~$|N| > |P|$, we split $N$ and the arguments are symmetrical.

Importantly, this algorithm \emph{always finds a solution if there exists one}, no matter how the \bs solver is implemented.
Indeed, if the \bs solver never returns a valid solution, we end up in the base case for divide and conquer (line~\ref{line:baseCaseDC} in Algorithm~\ref{algo:divideAndConquer}), where both $P$ and $N$ contain a single element --- assume $P = \set{p}$ and~$N = \set{n}$.
Per Lemma~\ref{lem:solution_exists}, an~$\bsSet$ such that $p\in \bsSet$ and $n\in \bsSet$ necessarily exists if there is a solution, and failure to exist immediately indicates that there is no solution to the \bs problem.
This shows that the algorithm always returns a valid solution if there exists one, no matter how the \bs solver is implemented.
However, there is no guarantee of minimality.

Although the idea of this algorithm originated from~\cite{Valizadeh.Fijalkow.ea:2024} for \LTLf learning, we introduce two optimizations.

The first is to use the solution from the first subproblems to simplify the second subproblems further;
if~$\bsFormula_1$ is the returned solution to the instance $(P_1, N, \family_{\restriction P_1\cup N})$, it may already contain elements of $P_2$.
It therefore suffices to solve the instance $(P_2 \setminus \bsEval{\bsFormula_1}, N, \family_{\restriction (P_2 \setminus \bsEval{\bsFormula_1})\cup N})$ for the second subproblem.
This optimization means that the solution may depend on the order in which we solve the two subproblems.

The second is related to \bs: the divide and conquer only combines existing formulas with operators $\cup$ and $\cap$, and can therefore be seen as a \bs algorithm.
This means that once it is called, the other optimizations for \bs (reducing evaluations from tables to vectors, stronger collapse due to observational equivalence, domination reduction) can be applied.

\section{Algorithm for \bs} \label{sec:algorithmsBS}
To solve \bs, we propose a greedy algorithm based on \emph{beam search}.
The gist is to enumerate formulas in order of their weight but only keep a fixed number of ``best'' formulas of each weight.
The notion of ``best'' is again based on the cardinality of $\bsSat{\bsFormula}$, below called \emph{score}.
The number of formulas to keep for each weight is a hyperparameter \beamwidth, and the parameter \dcswitch is the depth at which we stop exploring and simplify the problem using divide-and-conquer.

We provide the full pseudocode of Beam Search in Algorithm~\ref{algo:memgreedy_bs}.
For each weight, a ``$\min$'' priority queue $\bspq$ stores the (at most) \beamwidth formulas of this weight with the largest score, where the priority queue allows for an efficient query of the stored formula with the lowest score.

\begin{algorithm}[tbph]
	\caption{Beam search algorithm for \bs}\label{algo:memgreedy_bs}
	\textbf{Hyperparameters}:
	\begin{itemize}
		\item \dcswitch (not shown below), the weight up to which to enumerate before splitting with \DandC;
		\item \beamwidth (below, $\bsLim$), the maximal number of promising formulas to store for each weight.
	\end{itemize}
	\begin{algorithmic}[1]
		\Procedure{AddBounded}{$\bsFormula, P, N, \bspq, \bsLim$}
		\If{$\mathsf{size}(\bspq) < \bsLim$ or $\bsScore(\bsFormula) > \bsScore(\min(\bspq))$}
		\State{Insert $\bsFormula$ into $\bspq$}
		\EndIf
		\If{$\mathsf{size}(\bspq) > \bsLim$}
		\State{Pop minimal element from $\bspq$}
		\EndIf
		\EndProcedure

		\Procedure{BeamSearch}{$P, N, \family, \bsLim$}
		\State{$\bspq_1, \ldots, \bspq_{\max_{\bsSet\in\family} \bsSize{\bsSet}} \gets$ empty priority queues}
		\For{$\bsSet\in \family$}
		\State{\textsc{AddBounded}($\bsSet, P, N, \bspq_{\bsSize{\bsSet}}, \bsLim$)}
		\EndFor

		\State $k \gets 2$
		\State $\bsMem \gets \bspq_1 \cup \bspq_2$
		\While{True}
		\For{$i, j \ge 1, i + j = k$}
		\For{$\bsFormula_1 \in \bspq_{i}, \bsFormula_2 \in \bspq_{j}$}
		\For{$\op \in \{\cup, \cap\}$}
		\State $\bsFormula \gets \bsFormula_1 \mathop{\op} \bsFormula_2$ \Comment{$\bsFormula$ has weight $k+1$}
		\If{there exists $\bsF' \in \bsMem$ s.t.\ $\bsF'$ and $\bsF$ are observationally equivalent}
		\State \textbf{continue}
		\EndIf
		\If{there exists $\bsF' \in \bsMem$ s.t.\ $\bsF \domin \bsF'$}
		\State \textbf{continue}
		\EndIf
		\If{$\bsEval{\bsFormula} = P$}
		\State \textbf{return} $\bsFormula$
		\EndIf
		\State \textsc{AddBounded}($\bsFormula, P, N, \bspq_{k+1}, \bsLim$)
		\EndFor
		\EndFor
		\EndFor
		\State $\bsMem \gets \bsMem \cup \bspq_{k+1}$
		\State $k \gets k + 1$
		\EndWhile
		\EndProcedure
	\end{algorithmic}
\end{algorithm}

\section{Experiments} \label{sec:experiments}
\begin{table*}[t]
	\centering
	\resizebox{0.99\textwidth}{!}{\begin{tabular}{lcccccl}\toprule
& \multicolumn{2}{c}{solved tasks (TO 60~s)} & \multicolumn{2}{c}{avg.\ time (s)} & \multicolumn{2}{c}{avg.\ size ratio}\\\cmidrule(lr){2-3}\cmidrule(lr){4-5}\cmidrule(lr){6-7}
	& Scarlet & \tool & Scarlet & \tool & Scarlet & \tool \\\midrule
All benchmarks & 11468 / 15595 & 14374 / 15595 & 4.21 & 2.05 & 1.15 & 1.01\\\bottomrule
Fixed formulas & 74 / 81 & 81 / 81 & 1.83 & 0.00 & 1.19 & 1.00\\
Ordered Sequence & 1188 / 2160 & 2160 / 2160 & 7.16 & 0.28 & 1.70 & 1.00\\
Subword & 532 / 1256 & 1182 / 1256 & 5.86 & 8.56 & 1.43 & 1.01\\
Subset & 1304 / 1800 & 1082 / 1800 & 10.58 & 3.69 & 1.17 & 1.00\\
Hamming & 3 / 500 & 304 / 500 & 42.90 & 5.93 & 1.65 & 1.16\\
Single counter & 17 / 46 & 22 / 46 & 7.79 & 1.38 & 1.40 & 1.05\\
Double counter & 7 / 18 & 7 / 18 & 0.39 & 0.00 & 1.09 & 1.00\\
Nim & 23 / 63 & 22 / 63 & 2.16 & 0.00 & 1.15 & 1.00\\
Random conjuncts & 7549 / 8172 & 8074 / 8172 & 2.13 & 0.35 & 1.06 & 1.00\\
Random Boolean combinations & 771 / 1499 & 1440 / 1499 & 8.21 & 7.04 & 1.24 & 1.08\\\bottomrule
\end{tabular}
}
	\caption{Number of tasks solved by \tool and Scarlet~\cite{Raha.Roy.ea:2024} per benchmark family. For each algorithm and benchmark, we also report the average time (using \emph{arithmetic mean}) and average size ratio (using \emph{geometric mean}) on tasks that do not time out. \tool is faster and returns smaller formulas on average, even as we average over the $2906$~tasks for which \tool returns a formula but Scarlet does not.}
	\label{table:boltVSscarlet}
\end{table*}%
\begin{figure*}[t]
	\centering
	\includegraphics[width=.999\textwidth]{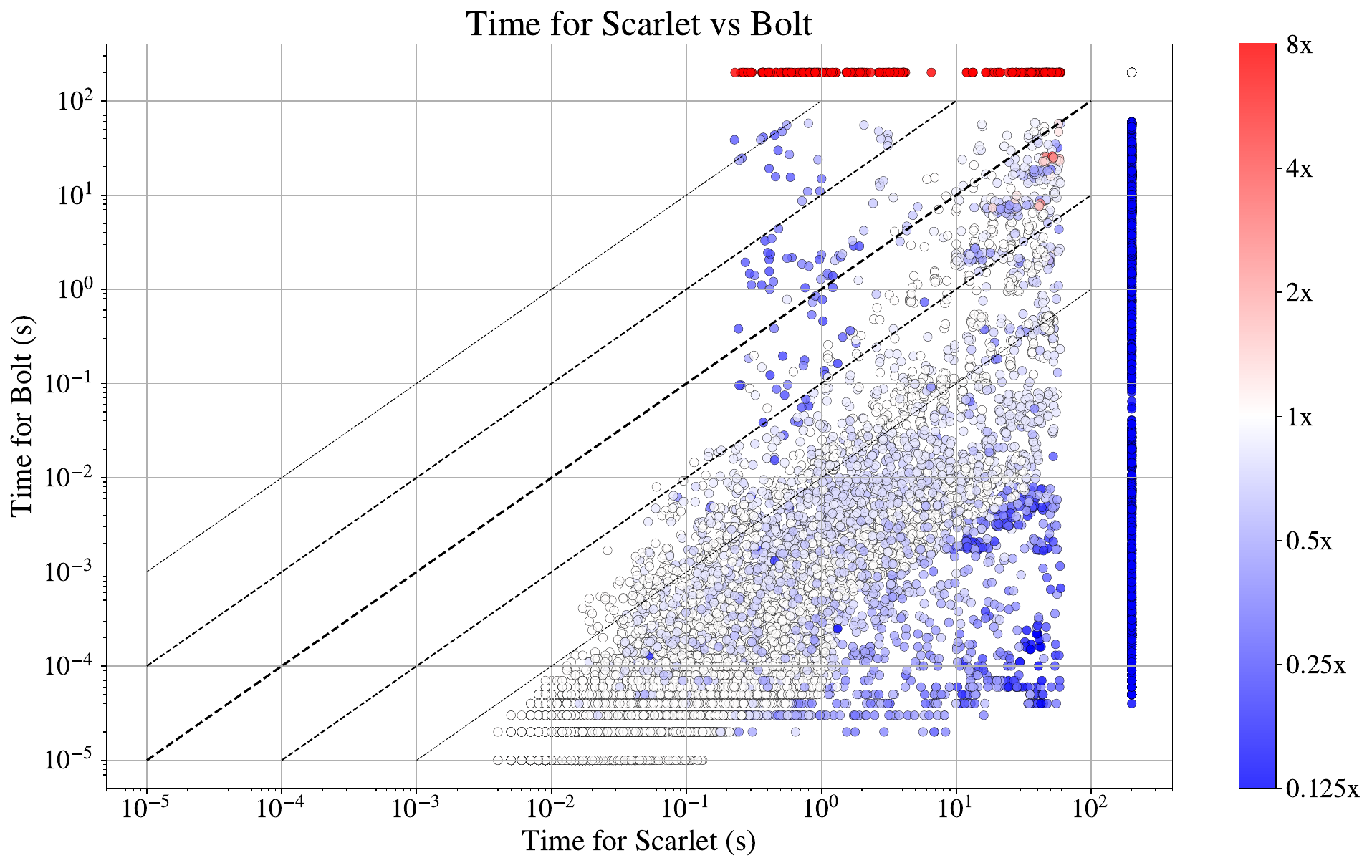}
	\caption{Comparison of running times (logarithmic axes) between Scarlet~\cite{Raha.Roy.ea:2024} and \tool. The diagonal lines represent ratios of time between the two tools: from bottom to top, \tool is $100$x faster, $10$x faster, equal, $10$x slower, and $100$x slower than Scarlet.
	Tasks that time out are assumed to be $200$~s so that they are apart in the plot. The timeout used in practice is $60$~s.
	The color represents the ratio of the size of the formula returned by \tool against the size of the formula returned by Scarlet.}
	\label{fig:comparison_plot}
\end{figure*}%

We perform experiments to address the following questions:
\begin{enumerate}
    \item What is \tool's performance against the state of the art?
    \item How is \tool's performance impacted by the switch to the Boolean Set Cover problem? Is it better than simply enumerating all LTL formulas?
\end{enumerate}
For each question, there are two (independent) metrics: wall-clock time and formula size.

\subsection{Benchmarks and State of the Art} \label{secsub:benchmarks}
As discussed in Section~\ref{sec:intro}, many tools have been constructed recently for $\LTLf$ learning.
As an independent contribution, we propose a consolidated benchmark suite of over $15,\!000$ $\LTLf$ learning tasks with difficulty ranging from very easy (solved by most existing tools) to very hard (not solved by any tool).
Our benchmark suite includes $10$~families, encompassing all publicly available benchmarks for $\LTLf$ learning we are aware of.
All but one family are inspired by existing benchmarks either directly for $\LTLf$ learning, or for related problems on $\LTLf$: in particular, we take advantage of SYNTCOMP's extensive collection of $\LTLf$ formulas~\cite{JacobsPABCCDDDFFKKLMMPR24}.
The benchmark suite is available at \benchmarkLink.

Each of the $10$~families consists of two scripts: $(i)$~a \textit{formula-generating script} for generating $\LTLf$ formulas, and $(ii)$~a \textit{task-generating script} for generating positive and negative traces for a formula in the family.
Following~\cite{Raha.Roy.ea:2022}, we implement a generic task-generating script based on compiling the $\LTLf$ formulas into deterministic automata (using the tool \emph{Spot}~\cite{Spot}) and sampling accepting and rejecting runs of a given length.
For some families, dedicated task-generating scripts yield finer and more challenging tasks.

Both scripts can be either deterministic or stochastic\footnote{In case of stochastic scripts, they are seeded for reproducibility.}, and use four parameters: trace length, number of atomic propositions, number of positive traces, and number of negative traces.
For generating the benchmark suite, we fixed trace lengths to be in $\set{16,32,64}$, and numbers of positive and negative traces to be in $\set{5,20,100}$.
As we will see, this produces challenging benchmarks --- should future progress move so far as to solve them all efficiently, changing these parameters will yield even harder benchmarks.

Let us briefly describe the $10$~families.
\subparagraph*{Fixed formulas.} The very first benchmarks were proposed in~\cite{Neider.Gavran:2018} and inspired by the seminal paper~\cite{Dwyer.Avrunin.ea:1999}, which identified a set of commonly used $\LTL$ formulas. The formulas being rather small, the corresponding tasks are typically easy.

\begin{table*}[t]
	\centering
	\resizebox{0.99\textwidth}{!}{\begin{tabular}{lcccccl}\toprule
& \multicolumn{2}{c}{solved tasks (TO 60~s)} & \multicolumn{2}{c}{avg.\ time (s)} & \multicolumn{2}{c}{avg.\ size ratio}\\\cmidrule(lr){2-3}\cmidrule(lr){4-5}\cmidrule(lr){6-7}
	& GPU & \tool & GPU & \tool & GPU & \tool \\\midrule
All benchmarks & 10948 / 11038 & 10259 / 11038 & 2.48 & 2.56 & 1.09 & 1.02\\\bottomrule
Fixed formulas & 54 / 54 & 54 / 54 & 0.78 & 0.00 & 1.06 & 1.00\\
Ordered Sequence & 1440 / 1440 & 1440 / 1440 & 0.84 & 0.40 & 1.04 & 1.00\\
Subword & 1242 / 1256 & 1182 / 1256 & 8.91 & 8.56 & 1.29 & 1.01\\
Subset & 1174 / 1200 & 710 / 1200 & 2.06 & 3.71 & 1.10 & 1.00\\
Hamming & 308 / 320 & 284 / 320 & 12.83 & 6.15 & 1.16 & 1.17\\
Single counter & 22 / 31 & 15 / 31 & 7.62 & 2.01 & 1.22 & 1.07\\
Double counter & 12 / 12 & 5 / 12 & 5.47 & 0.00 & 1.03 & 1.00\\
Nim & 28 / 42 & 15 / 42 & 1.67 & 0.00 & 1.06 & 1.00\\
Random conjuncts & 5442 / 5448 & 5366 / 5448 & 0.80 & 0.42 & 1.04 & 1.00\\
Random Boolean combinations & 1226 / 1235 & 1188 / 1235 & 3.14 & 7.47 & 1.17 & 1.09\\\bottomrule
\end{tabular}
}\caption{Number of tasks solved by \tool and the GPU algorithm~\cite{Valizadeh.Fijalkow.ea:2024}, per benchmark family.
	We omit in this comparison the $4557$ tasks that the GPU algorithm cannot handle due to technical constraints (trace lengths of size $64$ or more).
	See Table~\ref{table:boltVSscarlet} for an explanation of the columns.
	The GPU algorithm returns formulas for more tasks, but \tool is competitive in terms of time and size ratio over tasks solved by both algorithms.}
	\label{table:boltVSGPU}
\end{table*}%
\subparagraph*{Hamming.} A family of very hard tasks was created in~\cite{Valizadeh.Fijalkow.ea:2024}. It consists of a single positive trace and many negative traces obtained by changing a few bits from the positive one.

\subparagraph*{Deterministic parametric families.} The next three families are based on parametric families of $\LTLf$ formulas:
\begin{itemize}
	\item \emph{OrderedSequence} considers formulas $a_0 \lU (a_1 \lU (a_2 \dots))$, called ``Uright'' in~\cite{JacobsPABCCDDDFFKKLMMPR24},
	\item \emph{Subword}~\cite{Raha.Roy.ea:2022} considers formulas $\lF(a_0 \wedge \lX \lF(a_1 \wedge \dots))$ that express the existence of a fixed subword,
	\item \emph{Subset}~\cite{Raha.Roy.ea:2022} considers formulas $\lF a_0 \wedge \lF a_1 \wedge \lF a_2 \wedge \dots$ that express the presence of a subset of atomic propositions.
\end{itemize}

\subparagraph*{SYNTCOMP families.} The next three families were defined for the SYNTCOMP competition:
\emph{SingleCounter} and \emph{DoubleCounter} reason on counter increments, and \emph{Nim} encodes the classical eponymous game.

\subparagraph*{Randomized families.} The last two families are sampling methods for $\LTLf$ formulas:
\begin{itemize}
	\item \emph{RandomConjunctsFromBasis} considers conjunctions of a fixed number of formulas (after applying random permutations of the variables)~\cite{10.1007/3-540-48683-6_23},
	\item \emph{RandomBooleanCombinationsofFactors} considers random Boolean combinations of so-called \emph{patterns}. 
	This is inspired by decompositions of $\LTLf$ formulas into patterns introduced in~\cite{Raha.Roy.ea:2022,Fijalkow.Lagarde:2021,Mascle.Fijalkow.ea:2023}.
	A \emph{pattern} is an \LTLf formula in some normal form using $\lF$ and $\lX$; for instance,
	$\lF(p \land \lX(q \land \lX r))$, where $p$, $q$, and~$r$ are atomic propositions.
	This family is introduced in our paper.
\end{itemize}

\subsection{Experiments} \label{secsub:experiments}

\subparagraph*{Comparison with Scarlet.}
Since Scarlet~\cite{Raha.Roy.ea:2022,Raha.Roy.ea:2024} is a CPU algorithm which improved over its competitors by a large margin, we take it as a reference point for the state of the art. Scarlet is an ``anytime'' algorithm, meaning that it outputs a stream of formulas, each of them solutions for the task but of decreasing size, and it guarantees that the final output is minimal (within the targeted ``directed fragment of $\LTLf$'') upon termination. When reporting on time for Scarlet, we report how long it took to output the last formula, and \emph{not} termination time, which is typically a lot longer.

Below, ``\tool'' refers to our \vfb implementation with \emph{beam search} as a \bs solver, with \LTLBSSwitch $= 8$, \beamwidth $= 100$, and \dcswitch $= 70$.\footnote{%
We have performed experiments on a small training set (containing $20$\% of the data) to determine good hyperparameters; we do not include these experiments here.}
The most critical hyperparameter is \LTLBSSwitch, for which~$8$ appears to achieve the best trade-off between running time and size ratio.
Higher values of \LTLBSSwitch
lead to more timeouts: more time is spent enumerating small \LTLf formulas, but less time is spent exploring large formulas with a \bs solver.
With smaller values, about the same number of benchmarks are solved, but
the \bs phase has access to fewer \LTLf formulas, leading to longer solves and formulas of much larger sizes.

\begin{figure*}[ht]
	\centering
	\includegraphics[width=0.99\textwidth]{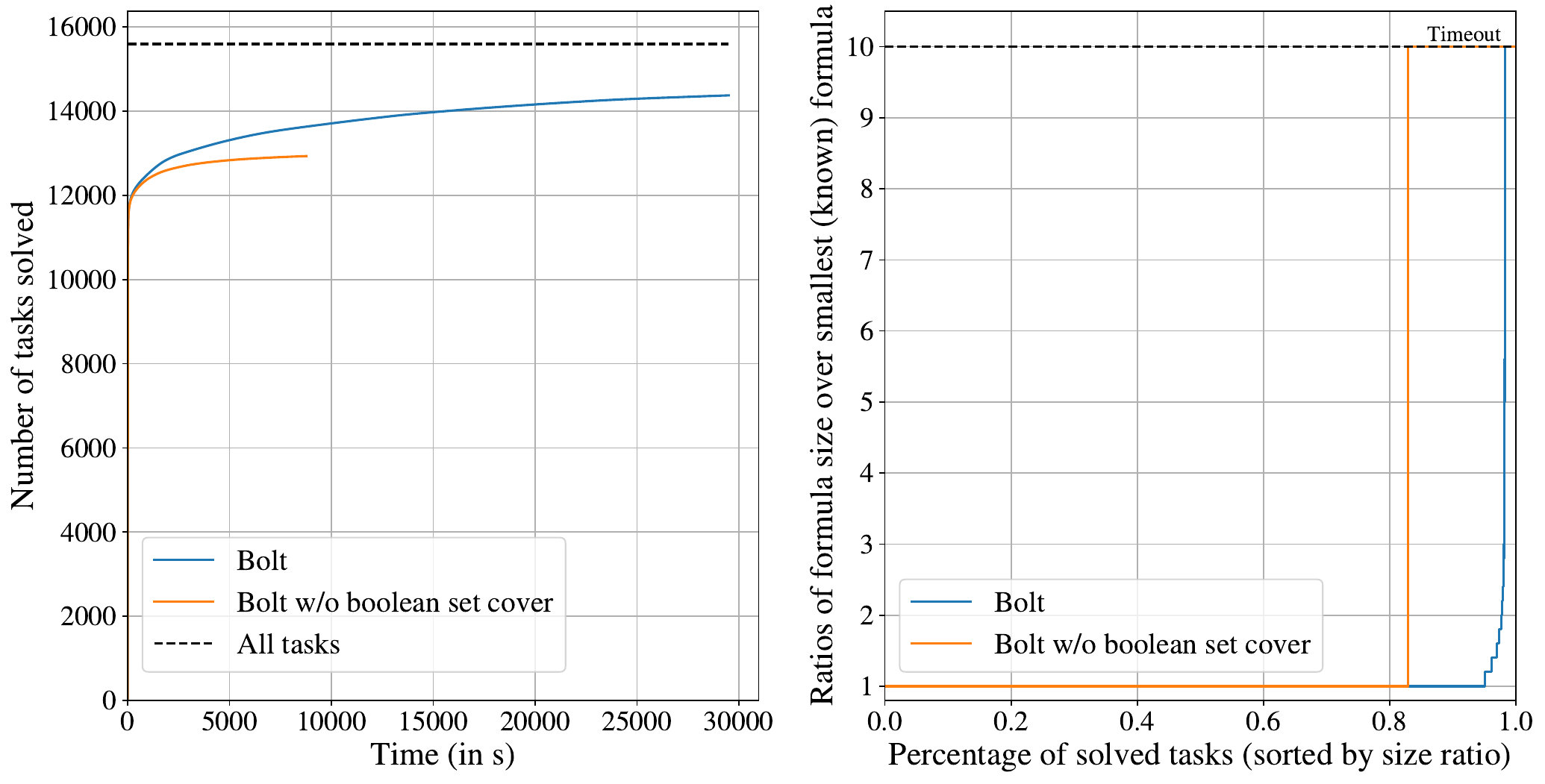}
	\caption{Ablation study: removing \bs.}
	\label{fig:ablation_study}
\end{figure*}

When comparing formula size, we use \emph{size ratio}, which is the ratio between the size of the returned formula and the shortest formula found by all our runs using Scarlet, Bolt, and the GPU algorithm (with a few combinations of hyperparameters).
We use size ratio instead of formula size to have a uniform measure over all tasks, as different tasks may admit minimal formulas of widely different sizes.
No algorithm uses randomization and we observe negligible performance shifts on runs over the same tasks.

\tool has been implemented in Rust using standard libraries (in particular, no parallel architectures).
The implementation is available at \boltLink.
Our implementation includes two other \bs solvers, but we do not report on them here as they do not yield better results than the beam-search algorithm we described.
The experiments were conducted on identical nodes of the \emph{Grid'5000} cluster, running Debian GNU/Linux~5.10.0-34-amd64.
The hardware configuration includes an Intel(R) Xeon(R) Gold~5320~CPU~@~2.20GHz and 384GB of RAM.
\tool is implemented in the Rust programming language. Our code was compiled in \texttt{release} mode, using the \texttt{rustc} Rust compiler, version 1.87.0 (\texttt{17067e9ac}, 2025-05-09).

We compare \tool against Scarlet~\cite{Raha.Roy.ea:2022} in Table~\ref{table:boltVSscarlet} and Figure~\ref{fig:comparison_plot}.
In Table~\ref{table:boltVSscarlet}, we see that, except for the \emph{Subset} family (where one could argue that the formulas used to generate traces correspond to the fragment of formulas manipulated by Scarlet) and the small \emph{Nim} family, \tool performs better on all benchmarks families: it returns a solution for more tasks, with a better average running time and smaller average size ratio.
Figure~\ref{fig:comparison_plot} illustrates that \tool returns solutions more than $100$x faster over $70$\% of the benchmarks, with formulas of smaller or equal size in $98$\% of the cases (most of the formulas in the remaining $2$\% are in the \emph{Subset} family).

\subparagraph*{Comparison with the GPU algorithm.}
In 2024, article~\cite{Valizadeh.Fijalkow.ea:2024} targeted a new category: GPU-based tools, meaning that the implementation specifically takes advantage of the specificities of GPUs (mainly, massively parallel computations). 

We provide a brief comparison with the GPU-based algorithm of~\cite{Valizadeh.Fijalkow.ea:2024} in Table~\ref{table:boltVSGPU}.
A few details make this comparison difficult: $(i)$~the GPU implementation is obviously highly parallelized, which is not the case for our CPU implementation \tool; $(ii)$~the current version of the GPU algorithm has technical limitations, such as not being able to handle tasks where the trace length is $\ge 64$, which prevents us from running it on $4557$ of our new benchmark tasks; $(iii)$~the GPU algorithm always requires about $0.75$~s to start, which means that \tool is much faster on small tasks (such as the \emph{Fixed formulas} benchmarks); $(iv)$~the set of $\LTLf$ operators used by the GPU implementation is smaller than the one we use, which means that comparing the size of formulas makes little sense. 

\subparagraph*{Ablation Study.}
To evaluate the impact of \bs, we compare \tool with the raw \vfb (i.e., \tool without \bs) in Figure~\ref{fig:ablation_study}.
The raw \vfb solves fewer tasks, which shows that switching to \bs helps solve more hard tasks, while sacrificing little w.r.t.\ the size of the formulas.
Observe in the right plot that the raw \vfb either finds a minimal formula or times out, as expected.

\section{Perspectives} \label{sec:perspectives}

The first contribution of this paper is a framework for combining $\LTLf$ learning with \bs.
We proposed a new algorithm for \bs and, through experimental analyses, we showed that our new tool \tool greatly improves over the state of the art for CPU algorithms.

These results yield evidence toward our main thesis: \BoolSynth is a fundamental problem of independent interest, which has a lot of potential applications beyond $\LTLf$ learning. Our framework is generic; for instance, our tool (through the \vfb) can be adapted to any kind of temporal operators, and \BoolSynth can be used for any logic or specification language with disjunctions and conjunctions. Natural candidates include regular expressions~\cite{Valizadeh.Berger:2023}, Boolean circuits~\cite{Chowdhury.Romanelli:2024}, and Bit-vector programs~\cite{Ding.Qiu:2024}.

Our paper provides a first step toward efficiently and approximately solving \BoolSynth. We believe there are algorithms that achieve better trade-offs between compute time and size of learned formulas.
In particular, we expect our algorithms to be highly parallelizable; we leave for future work an implementation of our framework for GPUs and a comparison with the CPU implementation.
We also contributed a large consolidated benchmark suite, which is much needed to push further the theory and practice of $\LTLf$ learning and \BoolSynth.

\bibliography{references}

\newpage
\appendix
\section{Pseudocode for the \vfb (Section~\ref{sec:ltl_learning})} \label{app:pseudocodesLTL}

For completeness, we show in Algorithm~\ref{algo:vanilla_combinatorial_search} the complete pseudocode for the \vfb (Section~\ref{secsub:ltl_implem}).
The way we enumerate all formulas of size $k+1$ is as follows: for a unary operator, we take all formulas of size $k$ in order to produce formulas of size $k+1$ and, for a binary operator, we iterate over all pairs of formulas whose sizes sum to $k$: $(1, k-1), (2, k-2), \ldots$

\begin{algorithm}[ht]
	\caption{\vfb}\label{algo:vanilla_combinatorial_search}
	\begin{algorithmic}[1]
		\Procedure{VFB}{$\AP, O_1, O_2, \CT$}
		\Comment{$O_n$ is the set of $n$-ary operators, $\CT$ takes a formula and computes its characteristic table}
		\State $\family_1 \gets \AP$
		\State $M \gets \left\{ \CT(a) \mid a \in \AP \right\}$ \Comment{Compute initial characteristic tables}
		\State $k \gets 1$
		\While{True}
		\State $\family_{k+1} \gets \emptyset$

		\For{$\op \in O_1$}
		\For{$\phi \in \family_{k}$}
		\State $\phi' \gets \op(\phi)$ \Comment{Size $k+1$}
		\If{$\phi'$ is a solution}
		\State \textbf{return} $\phi'$
		\ElsIf{there is no $\psi$ such that $\psi \sim \phi'$}
		\State $\family_{k+1} \gets \family_{k+1} \cup \{ \phi'\}$
		\State $M \gets M \cup \{ \CT(\phi')\}$
		\EndIf
		\EndFor
		\EndFor
		\For{$\op \in O_2$}
		\For{$i, j \geq 1, i + j = k$}
		\For{$\phi_1 \in \family_{i}, \phi_2 \in \family_{j}$}
		\State $\phi' \gets \op(\phi_1, \phi_2)$ \Comment{Size $k+1$}
		\If{$\phi'$ is a solution}
		\State \textbf{return} $\phi'$
		\ElsIf{there is no $\psi$ s.t.\ $\psi \sim \phi'$}
		\State $\family_{k+1} \gets \family_{k+1} \cup \{ \phi'\}$
		\State $M \gets M \cup \{ \CT(\phi')\}$
		\EndIf
		\EndFor
		\EndFor
		\EndFor
		\State $k \gets k + 1$ \Comment{No solution of size $k+1$}
		\EndWhile
		\EndProcedure
	\end{algorithmic}
\end{algorithm}

As discussed in Section~\ref{secsub:ltl_implem}, we show how to compute inductively the characteristic sequences of $\LTLf$ formulas on traces with a few bitwise operations:
\begin{enumerate}
	\item $\lX\phi$ is just a (left) shift by $1$ of the characteristic sequences of $\phi$, padding it with a $0$.
	\item $\lnot \phi$ is the bitwise negation of the characteristic sequences.
	\item $\phi \land \psi$ is the bitwise ``and'' of the characteristic sequences of the two formulas; the same can be done for $\phi \lor \psi$ with ``or''.
	\item $\lF \phi$ can be seen as $\phi \lor \lX \phi \lor \lX^2 \phi\lor \ldots{}$
	Since we work with finite traces, this converges to a fixed point in finite time.
	We can actually compute this efficiently with a \emph{logarithmic} number of shifts in the trace length: simply do shifts of powers of $2$ up to the length of the trace.
	For example, with the sequence $0000 0001 0000 0001$, first shift by $1$ and do an ``or'': it gives us $0000 0011 0000 0011$.
	Shift by $2$ and do an ``or'', we get: $0000 1111 0000 1111$.
	Now we shift by $4$ and do an ``or'', giving us only $1$'s.
	If instead we have $0000 0000 0000 0001$, we also need to shift by~$8$.
	The example shows that since $1$'s propagate locally, we only need to do this operation $\log \ell$ times, where $\ell$ is the trace length.
	\item $\lG \phi$ is equivalent to $\lnot \lF \lnot \phi$.
	\item $\phi \lU \psi$ can be done similarly to $\lF \phi$ since it is $\psi \lor (\phi \land \lX \psi) \lor (\phi \land \lX(\phi \land \lX \psi)) \lor \ldots{}$ Intuitively, first compute $\phi \land \lX \psi$ for all positions; it suffices to do an ``and'' on a shifted~$\psi$.
	Now, following the same protocol as $\lF \phi$, it suffices to do the same shifts and ``or'', but with the addition that before the ``or'', as $\phi$ needs to be true at all positions before, we must do an ``and'' with the shifted~$\phi$.
\end{enumerate}

We include Python code for the fast inductive evaluation of $\LTLf$ formulas, based on bitwise operations.
These algorithms describe the general idea of how to compute iteratively the characteristic sequences of formulas from existing ones.

\begin{minted}{python}
def X(phi):
  return phi << 1
\end{minted}

\begin{minted}{python}
def logic_not(phi):
  return ~phi
\end{minted}

\begin{minted}{python}
def logic_and(phi, psi):
  return phi & psi
\end{minted}

\begin{minted}{python}
def F(phi):
  out = phi
  # Assume trace lengths <= 64
  for shift in [1, 2, 4, 8, 16, 32]:
    out |= (out << shift)
  return out
\end{minted}

\begin{minted}{python}
def U(phi, psi):
  out = psi
  # Assume trace lengths <= 64
  for shift in [1, 2, 4, 8, 16, 32]:
    out |= ((out << shift) & phi)
    phi = (phi << shift)
  return out
\end{minted}

\end{document}